\documentclass[twoside]{article}

\usepackage{PRIMEarxiv}

\usepackage[utf8]{inputenc} 
\usepackage[T1]{fontenc}    
\usepackage{hyperref}       
\usepackage{url}            
\usepackage{booktabs}       
\usepackage{amsfonts}       
\usepackage{nicefrac}       
\usepackage{microtype}      
\usepackage{lipsum}
\usepackage{amsthm}
\newtheorem{theorem}{Theorem}
\newtheorem{proposition}{Proposition}

\newtheorem{remark}{Remark}
\usepackage{fancyhdr}       
\usepackage{graphicx}       
\usepackage{amsmath} 
\graphicspath{{media/}}     
\fancypagestyle{plain}{%
  \fancyhf{}%
  \fancyfoot[RO,LE]{\thepage}%
}

\title{Prescriptive SVD-Inspired Attention\\ via Spectral Energy Retention
\thanks{
Accepted for publication in Transactions on Machine Learning Research (TMLR), 2026.
\url{https://openreview.net/forum?id=LZBWqyWNxS}
}
}

\author{\small
  \textbf{Vasileios Arampatzakis, Vasileios Sevetlidis, George Pavlidis} \\
  Athena Research Center, Greece \\
  \texttt{\{vasilis.arampatzakis, vasiseve, gpavlid\}@athenarc.gr} \\
}

\begin{document}
\maketitle
\thispagestyle{plain}

\begin{abstract}
Self-attention is central to modern Transformer architectures, but its dense dot-product formulation makes it difficult to identify which internal directions are structurally important and which can be modified without disrupting the model. SVD-Inspired Attention (SVDA) addresses part of this problem by introducing a learned diagonal spectrum into the query-key score interaction, making latent attention directions explicitly inspectable through indicators such as spectral entropy, effective rank, sparsity, alignment, selectivity, and perturbation response. This paper examines the transition from diagnostic interpretation to operational intervention. A diagnosis--intervention--verification framework is proposed, and one intervention is evaluated: spectral energy retention in the attention-score pathway. Across FashionMNIST, CIFAR-10, CIFAR-100, and Food-101, the $\rho=0.90$ prescription removes 24.5--53.7\% of score directions, reduces parameters by 2.6--4.3\%, and reduces estimated MACs by 2.8--5.4\%. The paired mean accuracy change of the dimension-reduced model ranges from $-0.03$ to $+0.05$ percentage points over three seeds. These results support SVDA as an intrinsically interpretable attention mechanism whose learned spectrum exposes an operational coordinate system for deterministic and verifiable modification of attention-score formation.
\end{abstract}


\section{Introduction}
\label{sec:introduction}

Transformer architectures have reshaped modern machine learning by replacing fixed local processing patterns with adaptive self-attention. In vision, language, dense prediction, and multimodal learning, attention has become a dominant mechanism for routing information across tokens, spatial regions, or semantic units. Yet the same flexibility that makes attention powerful also makes it difficult to interpret. Standard scaled dot-product attention forms dense interactions between projected queries and keys, but these interactions do not explicitly separate direction, magnitude, spectral concentration, or latent dimensional usage. As a result, attention maps may be visually inspectable but structurally opaque.

SVD-Inspired Attention (SVDA) was introduced as a geometrically structured alternative to standard dot-product attention~\cite{arampatzakis2025geometry}. Its core idea is to replace unconstrained query-key interaction with a formulation in which row-normalized query/key projections interact through a learned diagonal spectral modulation matrix. This structure draws inspiration from the singular-value decomposition, where directional information and spectral importance are represented separately. In subsequent work, SVDA was applied to monocular depth estimation~\cite{arampatzakis2026svda_mde} and image classification~\cite{arampatzakis2026svda_classification}, showing that the same spectral formulation can support intrinsic attention diagnostics across both dense prediction and classification settings. The central value of SVDA in these studies was \textit{diagnostic}. The learned spectrum made it possible to monitor spectral entropy, effective rank, spectral sparsity, angular alignment, attention selectivity, and perturbation robustness across layers, heads, and training epochs. These quantities provided a structured description of how attention organizes internally. They exposed whether a head used many or few latent directions, whether spectral energy was concentrated or diffuse, whether attention became more selective with depth, and whether attention distributions remained stable under small perturbations. However, diagnostic interpretability leaves an important question unanswered. Weak spectral directions, redundant heads, low-effective-rank layers, and perturbation-unstable attention maps should not remain merely descriptive diagnostics. They should provide evidence for targeted compression, pruning, regularization, or architectural adjustment. \textit{Interpretability becomes more useful when it supports intervention.} A diagnostic should not only describe model behavior, but ideally, it should also indicate where compression, pruning, regularization, or architectural adjustment may be justified.

In this work, \textit{prescription} refers to evidence-guided intervention: SVDA identifies structured directions in the attention-score operator and provides concrete candidates for verification.
A prescriptive extension of SVDA is developed, in which spectral diagnostics guide targeted interventions on the attention-score operator. The focus is on using the learned spectrum as an explicit control surface: low-energy directions, concentrated effective rank, perturbation-sensitive heads, and spectrally similar heads become concrete signals for pruning, compression, regularization, or redundancy analysis.
The proposed framework follows a \textit{diagnosis--intervention--verification} loop: train and diagnose an SVDA model, propose an intervention from the observed spectrum, and verify its predictive and structural effects. The formal analysis gives operator-level bounds for pruning and perturbation, and the empirical study evaluates spectral energy retention on FashionMNIST, CIFAR-10, CIFAR-100, and Food-101. The same framework also organizes effective-rank targeting, robustness regularization, and head-redundancy screening as spectrum-guided intervention modes. Matched random and largest-$\Sigma$ controls test the specificity of the learned spectral ordering.

The original SVDA formulation introduced geometry into attention. The present work examines whether that geometry can guide controlled modification of the attention-score pathway. The main contributions are listed below.
\begin{itemize}
\item A diagnosis--intervention--verification framework is introduced for SVDA, moving from passive inspection of attention spectra to actionable model intervention.
\item Several candidate interventions are defined as transformations of the SVDA attention operator, with spectral energy retention studied as the primary empirical instance.
\item Operator-level bounds relate spectral pruning to changes in pre-softmax scores and attention distributions, while effective rank and thresholded support are used as descriptive indicators of spectral concentration and candidate score-pathway dimensionality.
\item An energy-retention prescription is evaluated across four image-classification datasets, including a 224$\times$224 Food-101 setting, using masked and dimension-reduced score pathways plus repeated matched random masks, high-energy removal controls, a standard-attention magnitude baseline, and selected width-matched and ablation checks.
\item The study establishes a diagnosis--intervention--verification workflow in which attention interpretability directly informs a structurally realizable model edit.
\end{itemize}

\section{Related Work}
\label{sec:related_work}

\subsection{Attention Interpretability and Intrinsic Explanations}

The interpretability of attention mechanisms remains a central issue in Transformer research. Although attention maps are often visualized as explanations, several works have cautioned that raw attention weights should not automatically be treated as faithful explanations of model behavior~\cite{wiegreffe2019attention, bastings2020elephant}. In vision Transformers, Chefer et al.~\cite{chefer2021transformer} proposed relevance propagation beyond direct attention visualization, while Raghu et al.~\cite{raghu2021vit} analyzed how representations and attention evolve across depth. More recent interpretability-oriented approaches, such as prototype-based Vision Transformers~\cite{ma2024protovit} and interpretability-aware ViTs~\cite{qiang2023iavt}, attempt to make visual recognition more transparent through architectural or training-objective modifications. These approaches are important, but they either explain an already-trained attention mechanism or add auxiliary interpretability structures around it. SVDA follows a different route. The original SVDA formulation modified the attention operator itself by introducing row-normalized query/key projections and a learned diagonal spectral modulation matrix~\cite{arampatzakis2025geometry}. Subsequent SVDA studies demonstrated that this intrinsic structure exposes diagnostic indicators in both monocular depth estimation~\cite{arampatzakis2026svda_mde} and image classification~\cite{arampatzakis2026svda_classification}. The present work extends this line by asking whether those diagnostics can guide concrete interventions on the model.

\subsection{Structured, Low-Rank, and Spectral Attention}

Several Transformer variants introduce structure into attention for efficiency, stability, or improved representation. Linformer~\cite{wang2020linformer}, Performer~\cite{choromanski2020performer}, Nystr{\"o}mformer~\cite{xiong2021nystromformer}, and Linear Transformers~\cite{katharopoulos2020linear} reduce the cost of attention through low-rank, kernel, or linearized approximations. Sparse and routing-based methods, including Sparse Transformers~\cite{child2019sparse}, Reformer~\cite{kitaev2020reformer}, BigBird~\cite{zaheer2020bigbird}, Routing Transformers~\cite{roy2021routing}, and adaptively sparse Transformers~\cite{correia2019adaptively}, restrict token interactions to improve scalability. Closer to SVDA are approaches that introduce spectral or geometric structure into attention. Singularformer~\cite{wu2023singularformer} decomposes attention to reduce complexity, Primal-Attention~\cite{chen2023primal} uses asymmetric kernel SVD in primal representation, CosFormer~\cite{qin2022cosformer} replaces softmax attention with cosine-based structure, and SpecFormer~\cite{bo2023specformer} incorporates spectral decomposition ideas in Transformer-like architectures. Other analyses directly examine SVD structure in query--key interactions, for example by studying the singular structure of $W_Q^\top W_K$ in vision Transformers~\cite{pan2024dissecting_qk}. NormFormer~\cite{shleifer2021normformer} adds normalization to improve Transformer training stability. These methods demonstrate the value of structured attention, but they primarily target efficiency, stability, representational improvement, or post-hoc analysis of existing attention interactions.

Several recent works have also explored more explicit forms of spectral control in Transformer architectures. Dynamic Spectral Weighting modulates attention heads using spectral characteristics of their outputs~\cite{huang2025dynamic_spectral_weighting}. Spectral Conditioning of Attention studies the Jacobian conditioning of attention blocks and modifies spectral properties of attention layers to improve stability and performance~\cite{saratchandran2026spectral_conditioning}. Frequency-domain spectral attention methods filter the attention score matrix through learnable spectral masks~\cite{huang2026spectral_attention}. These works indicate a broader interest in using spectral structure not only for analysis, but also for controlling attention behavior. However, they differ from SVDA in that SVDA introduces an explicit learned diagonal spectrum inside the query-key interaction itself, making latent directional importance directly inspectable and intervention-ready.

\subsection{SVDA as Diagnostic}

SVDA defines the attention operator as
\begin{equation}
\label{SVDA_attention}
A = \operatorname{softmax}\left(\frac{\left(Q\Sigma\right) K^{\top}}{\sqrt{d_k}}\right)
\end{equation}
where $Q,K\in\mathbb{R}^{n\times d_k}$ are row-normalized query and key projections, and $\Sigma = \operatorname{diag}(\sigma_1,\ldots,\sigma_{d_k})$ is a learnable diagonal modulation matrix. The term ``SVD-inspired'' refers to this separation between directional coordinates and diagonal spectral modulation; the model does not compute an SVD of the attention matrix, and the learned coefficients are not constrained to be singular values. In particular, $\sigma_r$ may be signed, while the energy-retention rule uses $\sigma_r^2$ only to define a nonnegative within-head ordering. This formulation separates directional alignment from spectral importance and enables indicators such as spectral entropy, effective rank, spectral sparsity, angular alignment, selectivity, and perturbation robustness. The role of $\Sigma$ is to reweight latent attention directions before the query-key interaction is converted into an attention distribution. The monocular depth estimation study integrated SVDA into DPT and showed that the same indicators expose depth-wise and training-time organization of attention in dense prediction~\cite{arampatzakis2026svda_mde}. The image-classification study adapted SVDA to ViTs and used the indicators to analyze attention structure across FashionMNIST, CIFAR-10, CIFAR-100, and ImageNet-100~\cite{arampatzakis2026svda_classification}. These studies established SVDA as an intrinsically diagnostic attention mechanism. They showed that the learned spectrum is not merely an additional parameterization, but a compact representation of how attention allocates latent directional capacity. The present work builds directly on this diagnostic foundation and makes the transition from diagnosis to intervention. Once weak directions, redundant heads, low-rank behavior, or unstable attention patterns are identified, the model can be modified in a structured way.

\subsection{Transformer Pruning, Compression, and Head Redundancy}

Model pruning and compression aim to reduce parameter count, computation, or memory cost while preserving predictive performance. In Transformers, attention heads, tokens, projection dimensions, and intermediate representations are natural targets for pruning. Early work showed that many attention heads can be removed with limited loss in performance~\cite{michel2019heads, voita2019analyzing}. In vision Transformers, token pruning and adaptive computation methods such as DynamicViT~\cite{rao2021dynamicvit}, EViT~\cite{liang2022evit}, and TokenLearner~\cite{ryoo2021tokenlearner} reduce redundant token processing. More recent pruning frameworks use stronger sensitivity, explanatory, latency, or curvature information. X-Pruner learns explainability-aware masks to prune vision Transformer units~\cite{yu2023xpruner}, while multidimensional pruning methods such as MDP optimize across several prunable axes, including query/key dimensions, heads, embeddings, and blocks under latency constraints~\cite{sun2025mdp}. HEART-ViT uses Hessian-guided dynamic attention and token pruning to jointly reason about token and head redundancy~\cite{uddin2025heartvit}. More broadly, adversarial pruning studies show that compression and robustness can be treated jointly, with pruning decisions evaluated not only by accuracy retention but also by their effect on robustness under adversarial perturbations~\cite{piras2025adversarial}. The present work is related to this pruning literature but differs in its source of evidence and its intervention target. Standard pruning methods often use magnitude, sensitivity, loss approximation, learned gates, explanatory masks, or latency-constrained architecture search. SVDA-guided pruning uses the learned attention spectrum itself. Since each diagonal entry of $\Sigma$ directly scales one latent score direction, spectral sparsity and effective rank provide operator-level evidence about which directions may be weak, redundant, or compressible. Thus, SVDA-guided pruning is a spectrum-grounded intervention rule in which coefficient energy is treated as an explicit operator-level signal and then tested against functional controls.

The compression target considered here is narrower than those of token-, head-, or multidimensional-pruning methods. DynamicViT, EViT, and TokenLearner reduce the number of tokens processed by later layers, head-pruning methods remove complete attention heads, and multidimensional methods may jointly change several architectural axes. SVDA energy retention instead reduces selected coordinates only in the query--spectrum--key score pathway and leaves the token sequence, value pathway, and output projection unchanged. Consequently, reductions in parameters or MACs are not directly comparable unless the methods are evaluated under a matched total-compute or parameter budget. These approaches are complementary: SVDA energy retention targets latent score directions inside an explicitly parameterized score operator, whereas token, head, and multidimensional pruning alter different structural axes of the Transformer.

\subsection{Robustness and Stability of Attention}

Robustness in attention-based models concerns the stability of predictions and internal representations under input perturbations, adversarial changes, or distribution shifts. Prior work has studied adversarially robust attention~\cite{kitada2021attention}, semantic perturbation robustness~\cite{munakata2022verifying}, and fine-grained sensitivity of dot-product self-attention~\cite{havens2024finegrained}. In the SVDA line, perturbation robustness was introduced as an attention-level diagnostic measuring the Frobenius change in attention maps under small input perturbations~\cite{arampatzakis2025geometry}. The same diagnostic was later used in depth estimation and classification to inspect stability across layers and epochs~\cite{arampatzakis2026svda_mde, arampatzakis2026svda_classification}. This paper treats robustness diagnostics as actionable signals. If a layer or head shows high perturbation response, SVDA makes it possible to inspect whether instability is associated with the learned spectrum, projection geometry, or attention distribution. The theoretical analysis below further shows that attention-score sensitivity is controlled in part by the spectral norm of $\Sigma$, giving a formal basis for robustness-guided intervention.

\section{From Diagnostic to Prescriptive SVDA}
\label{sec:diagnostic_to_prescriptive}

SVDA exposes structure inside the attention-score mechanism. The present work addresses the use of this structure as evidence for controlled intervention. We argue that the explicit spectral parameterization of SVDA makes attention not only diagnosable but also actionable. Since the diagonal entries of $\Sigma$ determine the contribution of individual latent directions, interventions on $\Sigma$ induce direct and analyzable changes in the attention operator. This observation motivates a prescriptive extension of SVDA, in which spectral diagnostics are used to guide pruning, compression, robustness tuning, and architectural adaptation.

The term \emph{prescriptive} emphasizes targeted, evidence-guided intervention. Low entropy, low effective rank, or high sparsity are meaningful in relation to a layer, task, and verification protocol. Prescriptive SVDA uses spectral diagnostics as concrete evidence for targeted intervention: a low-energy spectral direction may be pruned, head redundancy may be flagged for further analysis, an unstable layer may be regularized or fine-tuned, and a persistently low-rank region of the model may justify architectural compression. The proposed framework therefore follows a simple diagnosis--intervention--verification loop:
\begin{equation*}
\text{train}\;\longrightarrow\;\text{diagnose}\;\longrightarrow\;\text{intervene}\;\longrightarrow\;\text{verify}
\end{equation*}

This section formalizes how interventions on $\Sigma$ modify the SVDA score operator. The analysis focuses on operator-level consequences: score perturbations caused by spectral pruning and score sensitivity under changes in normalized query/key projections. Effective rank, thresholded support, and spectral similarity are introduced as descriptive quantities or candidate intervention signals that organize how the spectrum can guide subsequent verification.

\subsection{The SVDA Spectrum as a Control Surface}
\label{subsec:spectrum_control_surface}

Let
\begin{equation}
S_{\Sigma}=\frac{Q\Sigma K^{\top}}{\sqrt{d_k}}
\label{eq:svda_scores}
\end{equation}
denote the pre-softmax SVDA score matrix. The attention matrix is then
\begin{equation}
A_{\Sigma}=\operatorname{softmax}(S_{\Sigma})
\end{equation}
where the softmax is applied row-wise.

Because $Q$ and $K$ are row-normalized, each query vector $q_i$ and key vector $k_j$ satisfies
\begin{equation}
\|q_i\|_2 = \|k_j\|_2 = 1
\end{equation}
The corresponding score between tokens $i$ and $j$ is therefore
\begin{equation}
S_{\Sigma,ij} = \frac{1}{\sqrt{d_k}} q_i^{\top}\Sigma k_j = \frac{1}{\sqrt{d_k}} \sum_{r=1}^{d_k} \sigma_r q_{ir}k_{jr}
\label{eq:score_component}
\end{equation}
where \eqref{eq:score_component} makes explicit why $\Sigma$ can be treated as a control surface. Each $\sigma_r$ modulates the contribution of one latent direction to every query-key score. Increasing, shrinking, masking, or regularizing $\sigma_r$ has a direct effect on the attention operator. This differs from post-hoc interpretation of a trained attention map: in SVDA, the interpretable structure is part of the operator itself.

We therefore define an \emph{SVDA spectral intervention} as any transformation
\begin{equation}
\mathcal{T}:\Sigma\mapsto\widetilde{\Sigma}
\end{equation}
that replaces the learned spectrum $\Sigma$ by a modified spectrum $\widetilde{\Sigma}$, producing the intervened attention operator
\begin{equation}
A_{\widetilde{\Sigma}} = \operatorname{softmax} \left(\frac{Q\widetilde{\Sigma}K^{\top}}{\sqrt{d_k}}\right)
\label{eq:intervened_attention}
\end{equation}
Examples include threshold pruning, rank-targeted modulation, sparsity-inducing shrinkage, and robustness-oriented spectral smoothing. The remainder of this section formalizes these interventions and provides elementary guarantees that connect them to the behavior of the attention operator.

\subsection{Spectral Pruning}
\label{subsec:spectral_pruning}

A direct use of the learned SVDA spectrum is to identify latent directions whose contribution to attention is negligible. Given a threshold $\tau>0$, a binary mask is defined for each spectral direction $r=1,\ldots,d_k$ as
\begin{equation}
m_r^{(\tau)} =
\begin{cases}
1, & |\sigma_r|\geq \tau\\
0, & |\sigma_r|<\tau
\end{cases}
\label{eq:mask}
\end{equation}
and the pruned spectrum
\begin{equation}
\Sigma_{\tau} = \operatorname{diag}\left(m_1^{(\tau)}\sigma_1,\ldots, m_{d_k}^{(\tau)}\sigma_{d_k} \right)
\label{eq:pruned_sigma}
\end{equation}
The resulting pruned SVDA attention is
\begin{equation}
A_{\Sigma_{\tau}} = \operatorname{softmax} \left( \frac{Q\Sigma_{\tau}K^{\top}}{\sqrt{d_k}} \right)
\label{eq:pruned_attention}
\end{equation}

The intuition is simple: if a spectral direction has very small magnitude, then removing it should only weakly perturb the attention scores. The following result makes this precise at the pre-softmax level.

\begin{theorem}[Bounded score perturbation under spectral pruning]
\label{thm:pruning_score_bound}
Let $Q,K\in\mathbb{R}^{n\times d_k}$ have row-normalized rows, and let $\Sigma$ and $\Sigma_{\tau}$ be defined as in \eqref{eq:pruned_sigma}. Let
\begin{equation}
S_{\Sigma}=\frac{Q\Sigma K^{\top}}{\sqrt{d_k}},\qquad S_{\Sigma_{\tau}}= \frac{Q\Sigma_{\tau}K^{\top}}{\sqrt{d_k}}
\end{equation}
Then, for every pair of tokens $i,j$,
\begin{equation}
|S_{\Sigma,ij}-S_{\Sigma_{\tau},ij}| \leq \frac{\|\Sigma-\Sigma_{\tau}\|_2}{\sqrt{d_k}}
\label{eq:score_bound_general}
\end{equation}
Moreover, since $\Sigma-\Sigma_{\tau}$ is diagonal and contains only pruned entries,
\begin{equation}
|S_{\Sigma,ij}-S_{\Sigma_{\tau},ij}| \leq \frac{\tau}{\sqrt{d_k}}
\label{eq:score_bound_tau}
\end{equation}
\end{theorem}

\begin{proof}
For any token pair $(i,j)$, let $q_i,k_j\in\mathbb{R}^{d_k}$ denote the column-vector representations of the corresponding rows of $Q$ and $K$. Since the rows of $Q$ and $K$ are normalized, $\|q_i\|_2=\|k_j\|_2=1$. The $(i,j)$ entry of the score matrix is
\begin{equation}
S_{\Sigma,ij} = \frac{1}{\sqrt{d_k}}q_i^{\top}\Sigma k_j
\end{equation}
Similarly,
\begin{equation}
S_{\Sigma_{\tau},ij} = \frac{1}{\sqrt{d_k}}q_i^{\top}\Sigma_{\tau} k_j
\end{equation}
Therefore,
\begin{equation}
S_{\Sigma,ij}-S_{\Sigma_{\tau},ij} = \frac{1}{\sqrt{d_k}} q_i^{\top}(\Sigma-\Sigma_{\tau})k_j
\end{equation}
By the definition of the spectral norm,
\begin{equation}
|q_i^{\top}(\Sigma-\Sigma_{\tau})k_j| \leq \|q_i\|_2 \|\Sigma-\Sigma_{\tau}\|_2 \|k_j\|_2
\end{equation}
Using $\|q_i\|_2=\|k_j\|_2=1$ gives \eqref{eq:score_bound_general}. Since $\Sigma-\Sigma_{\tau}$ is diagonal and contains only the pruned entries, its spectral norm is the maximum absolute removed diagonal entry:
\begin{equation}
\|\Sigma-\Sigma_{\tau}\|_2 = \max_{r:m_r^{(\tau)}=0}|\sigma_r| \leq \tau 
\end{equation}
Substitution yields \eqref{eq:score_bound_tau}.
\end{proof}

Theorem~\ref{thm:pruning_score_bound} provides the first formal justification for SVDA-guided pruning. It shows that pruning is controlled by the magnitude of the removed spectral directions. Thus, when the SVDA diagnostic identifies consistently weak spectral components, their removal has a bounded effect on every pre-softmax attention score. Because downstream layers may amplify or compensate for score changes, the bound is best understood as a controlled-operator statement. It justifies the intervention as a structured modification of attention scores.

The experimental rule retains spectral energy rather than applying a fixed threshold. The corresponding bound follows directly. Let $\mathcal{K}$ be any retained set satisfying
\begin{equation}
\sum_{r\in\mathcal{K}}\sigma_r^2 \geq \rho\sum_{r=1}^{d_k}\sigma_r^2,
\end{equation}
and let $\Sigma_{\mathcal{K}}$ set entries outside $\mathcal{K}$ to zero. Then
\begin{equation}
\|\Sigma-\Sigma_{\mathcal{K}}\|_F^2
=\sum_{r\notin\mathcal{K}}\sigma_r^2
\leq (1-\rho)\|\Sigma\|_F^2.
\end{equation}
Because the spectral norm is bounded by the Frobenius norm, Theorem~\ref{thm:pruning_score_bound} gives
\begin{equation}
|S_{\Sigma,ij}-S_{\Sigma_{\mathcal{K}},ij}|
\leq \frac{\sqrt{1-\rho}\,\|\Sigma\|_F}{\sqrt{d_k}}.
\label{eq:energy_retention_score_bound}
\end{equation}
This operator-level bound connects the tested retention rule directly to the resulting score perturbation.

\subsection{Attention Perturbation After Pruning}
\label{subsec:softmax_pruning}

The previous theorem bounds the perturbation of the pre-softmax score matrix. Since the attention matrix is obtained through a row-wise softmax, the same perturbation should also be related to controlled changes in the resulting attention distribution.

\begin{proposition}[Controlled attention perturbation under spectral pruning]
\label{prop:softmax_pruning_bound}
Let $A_{\Sigma}$ and $A_{\Sigma_{\tau}}$ be the row-wise softmax attention matrices generated from $S_{\Sigma}$ and $S_{\Sigma_{\tau}}$. If no individual attention score changes by more than $\epsilon$ after pruning
\begin{equation}
\|S_{\Sigma}-S_{\Sigma_{\tau}}\|_{\infty} \leq \epsilon
\end{equation}
where $\|\cdot\|_{\infty}$ denotes the entrywise maximum norm as
\begin{equation}
\|M\|_{\infty}=\max_{i,j}|M_{ij}|
\end{equation}
then for each attention row $i$
\begin{equation}
\|A_{\Sigma,i}-A_{\Sigma_{\tau},i}\|_1 \leq \epsilon
\label{eq:softmax_lipschitz_bound}
\end{equation}
That is, row-wise softmax is $1$-Lipschitz from $\ell_\infty$ to $\ell_1$. In particular, Theorem~\ref{thm:pruning_score_bound} yields
\begin{equation}
\|A_{\Sigma,i}-A_{\Sigma_{\tau},i}\|_1 \leq \frac{\tau}{\sqrt{d_k}}.
\end{equation}
\end{proposition}

\begin{proof}
For $p=\operatorname{softmax}(z)$, the Jacobian is $J(z)=\operatorname{diag}(p)-pp^\top$. For any vector $v$ with $\|v\|_\infty\leq 1$, $(Jv)_a=p_a(v_a-\sum_b p_bv_b)$, and hence $\|Jv\|_1$ is the mean absolute deviation of a random variable supported on $[-1,1]$. This is at most $1$, so $\|J(z)\|_{\infty\rightarrow 1}\leq 1$. Integrating the Jacobian along the line segment between the two score rows proves \eqref{eq:softmax_lipschitz_bound}. The final inequality follows from Theorem~\ref{thm:pruning_score_bound}.
\end{proof}

This proposition connects spectral pruning to the actual attention distribution. Pruning small spectral directions induces a bounded and analyzable attention-level change, which is the relevant guarantee for diagnosis-guided intervention on the score pathway.

\subsection{Effective Rank as a Controllable Attention Dimension}
\label{subsec:effective_rank_control}

SVDA also enables a continuous notion of how many latent directions are effectively used by an attention head. Given the diagonal spectrum $\Sigma=\operatorname{diag}(\sigma_1,\ldots,\sigma_{d_k})$, define the normalized spectral energy distribution
\begin{equation}
p_r = \frac{\sigma_r^2}{\sum_{s=1}^{d_k}\sigma_s^2}, \qquad r=1,\ldots,d_k
\label{eq:spectral_energy}
\end{equation}
assuming $\Sigma\neq 0$. The spectral entropy is
\begin{equation}
H(\Sigma) = -\sum_{r=1}^{d_k}p_r\log p_r
\label{eq:spectral_entropy}
\end{equation}
and the effective rank is
\begin{equation}
r_{\mathrm{eff}}(\Sigma) = e^{H(\Sigma)}
\label{eq:effective_rank}
\end{equation}
This quantity provides a differentiable proxy for the number of active spectral directions. It is not a hard rank in the linear-algebraic sense. Rather, it measures how concentrated or diffuse the spectral energy is.

\begin{remark}[Range and interpretation of effective rank]
\label{rem:effective_rank_bounds}
For any nonzero diagonal spectrum
$\Sigma\in\mathbb{R}^{d_k\times d_k}$, the normalized spectral
energy values in \eqref{eq:spectral_energy} form a probability
distribution. Consequently,
\begin{equation}
1 \leq r_{\mathrm{eff}}(\Sigma) \leq d_k.
\label{eq:rank_bounds}
\end{equation}
The lower endpoint occurs when all spectral energy is concentrated
in one direction, whereas the upper endpoint occurs when the energy
is uniformly distributed across all $d_k$ directions. These bounds
are a direct consequence of
$0\leq H(\Sigma)\leq\log d_k$ and are included to clarify the scale
of the diagnostic.
\end{remark}

This result formalizes the interpretation of effective rank as an attention-dimensionality indicator. In diagnostic SVDA, $r_{\mathrm{eff}}$ reveals whether a head uses many latent directions or only a few dominant ones. In prescriptive SVDA, the same quantity can be used to guide intervention. For example, a head with consistently low effective rank may be a candidate for dimensional compression, while a head whose effective rank collapses too early may indicate premature specialization.

One possible extension is to use effective rank as a training target through
\begin{equation}
\mathcal{L}_{\mathrm{rank}}=\left(r_{\mathrm{eff}}(\Sigma)-r^{\star}\right)^2,
\label{eq:rank_loss}
\end{equation}
where $r^{\star}$ denotes a chosen target effective rank. This objective defines a testable mechanism for encouraging spectral energy to occupy a specified range and provides one natural extension of the prescriptive SVDA framework.

\subsection{Spectral Sparsity and Compressibility}
\label{subsec:sparsity_compressibility}

While effective rank measures the soft dimensionality of spectral usage, spectral sparsity measures the proportion of latent directions that are effectively inactive. For a threshold $\epsilon>0$, define
\begin{equation}
P_{\epsilon}(\Sigma) = \frac{1}{d_k} \left| \{r:|\sigma_r|<\epsilon\} \right|
\label{eq:spectral_sparsity}
\end{equation}
The complementary quantity
\begin{equation}
C_{\epsilon}(\Sigma) = \left| \{r:|\sigma_r|\geq\epsilon\} \right|
\label{eq:compressible_dimension}
\end{equation}
is the number of active spectral directions above threshold. This number defines the effective compressed dimension of the head after thresholding.

\begin{remark}[Support induced by threshold pruning]
\label{rem:spectral_support}
Let $\Sigma_{\epsilon}$ be obtained by setting every spectral
coefficient satisfying $|\sigma_r|<\epsilon$ to zero. By construction,
the corresponding score matrix can be written as
\begin{equation}
S_{\Sigma_{\epsilon},ij}
=
\frac{1}{\sqrt{d_k}}
\sum_{r:|\sigma_r|\geq\epsilon}
\sigma_r q_{ir}k_{jr}.
\label{eq:thresholded_score_support}
\end{equation}
It therefore depends on at most
$C_{\epsilon}(\Sigma)$ active score directions. This identity defines
the structural support exposed by thresholding and identifies the score
coordinates retained by a thresholded intervention.
\end{remark}

The thresholded support provides an explicit candidate dimension for
the masked score operator. Converting that support into a physically
smaller implementation additionally requires specifying the
normalization and scaling used after coordinate selection. Predictive
performance and practical efficiency must be established empirically.

\subsection{Robustness-Guided Spectral Intervention}
\label{subsec:robustness_intervention}

SVDA also provides a natural way to reason about attention robustness. Let $x$ be an input and $x+\delta$ a perturbed version of the same input. The diagnostic perturbation response of an attention head can be measured as
\begin{equation}
\Delta_A(x,\delta) = \|A_{\Sigma}(x)-A_{\Sigma}(x+\delta)\|_F
\label{eq:perturbation_response}
\end{equation}
A large value of $\Delta_A$ indicates that the attention distribution changes substantially under a small input perturbation. This can be used diagnostically to identify unstable layers or heads. It can also motivate a targeted intervention, such as fine-tuning with a robustness penalty:
\begin{equation}
\mathcal{L}_{\mathrm{rob}} = \|A_{\Sigma}(x)-A_{\Sigma}(x+\delta)\|_F^2
\label{eq:robustness_loss}
\end{equation}
The following proposition states why the magnitude of $\Sigma$ matters for robustness.

\begin{proposition}[Spectral norm control of score sensitivity]
\label{prop:score_sensitivity}
Let
\begin{equation}
S_{\Sigma}(x) = \frac{Q(x)\Sigma K(x)^{\top}}{\sqrt{d_k}}
\end{equation}
be the SVDA score matrix. For a perturbed input $x+\delta$, assume that the normalized query and key matrices change to $Q'=Q+\Delta Q$ and $K'=K+\Delta K$. Then
\begin{equation}
\|S_{\Sigma}(x+\delta)-S_{\Sigma}(x)\|_F \leq \frac{\|\Sigma\|_2}{\sqrt{d_k}} \left( \|\Delta Q\|_F\|K'\|_2 + \|Q\|_2\|\Delta K\|_F \right)
\label{eq:sensitivity_bound}
\end{equation}
Thus, for fixed projection perturbations, attention-score sensitivity is controlled by the spectral norm of $\Sigma$.
\end{proposition}

\begin{proof}
We write
\begin{equation}
S_{\Sigma}(x+\delta)-S_{\Sigma}(x) = \frac{1}{\sqrt{d_k}} \left( Q'\Sigma K'^{\top} - Q\Sigma K^{\top} \right)
\end{equation}
Adding and subtracting $Q\Sigma K'^{\top}$ gives
\begin{equation}
Q'\Sigma K'^{\top} - Q\Sigma K^{\top} = (Q'-Q)\Sigma K'^{\top} + Q\Sigma(K'-K)^{\top}
\end{equation}
Therefore,
\begin{equation}
\|S_{\Sigma}(x+\delta)-S_{\Sigma}(x)\|_F \leq \frac{1}{\sqrt{d_k}} \left( \|\Delta Q\Sigma K'^{\top}\|_F + \|Q\Sigma \Delta K^{\top}\|_F \right)
\end{equation}
Using submultiplicativity of matrix norms,
\begin{equation}
\|\Delta Q\Sigma K'^{\top}\|_F \leq \|\Delta Q\|_F\|\Sigma\|_2\|K'\|_2
\end{equation}
and
\begin{equation}
\|Q\Sigma\Delta K^{\top}\|_F \leq \|Q\|_2\|\Sigma\|_2\|\Delta K\|_F
\end{equation}
Combining the two inequalities proves \eqref{eq:sensitivity_bound}.
\end{proof}
Proposition~\ref{prop:score_sensitivity} isolates one factor in score sensitivity: for fixed query and key perturbations, the bound scales with $\|\Sigma\|_2$. This motivates robustness-guided spectral intervention as a direct extension of the same control-surface view.

\subsection{Head Redundancy Through Spectral Similarity}
\label{subsec:head_redundancy}

SVDA also provides a compact representation of each attention head through its learned spectrum. For layer $\ell$ and head $h$, let
\begin{equation}
\boldsymbol{\sigma}_{\ell,h} = (\sigma_{\ell,h,1},\ldots,\sigma_{\ell,h,d_k})
\end{equation}
denote the vector of spectral coefficients. A simple spectral redundancy score between two heads $h_1$ and $h_2$ in the same layer is
\begin{equation}
R_{\ell}(h_1,h_2) = \frac{ |\boldsymbol{\sigma}_{\ell,h_1}|^{\top} |\boldsymbol{\sigma}_{\ell,h_2}| }{ \| \boldsymbol{\sigma}_{\ell,h_1}\|_2 \| \boldsymbol{\sigma}_{\ell,h_2}\|_2 }
\label{eq:head_redundancy}
\end{equation}
High spectral similarity is a compact screening signal for candidate head redundancy, while functional redundancy also depends on the query, key, value, and output projections. Empirical output similarity or verified head removal can then turn this screen into a redundancy decision, organizing a spectrum-guided workflow for head-level intervention.

\subsection{Summary of the Prescriptive SVDA Principle}
\label{subsec:prescriptive_summary}

The proposed extension of SVDA rests on a simple principle: because the diagonal spectrum $\Sigma$ is part of the attention-score operator, transformations of $\Sigma$ define explicit interventions on score formation. The formal analysis establishes two operator-level bounds relevant to this view:
\begin{itemize}
\item removing sufficiently small spectral coefficients produces a bounded perturbation of each pre-softmax attention score;
\item for fixed changes in the query and key projections, score sensitivity is bounded in part by the spectral norm of $\Sigma$.
\end{itemize}

The remaining quantities play descriptive or operational roles that structure the intervention space. Effective rank summarizes the concentration of spectral energy, thresholding defines the active support of the masked score operator, and spectral similarity screens candidate pairs of heads for later redundancy verification.

The empirical study below instantiates the framework through spectral energy retention. It tests whether low-energy directions can be masked and represented by a dimension-reduced query--spectrum--key pathway while retaining similar predictive behavior.

\section{Empirical Evaluation of Spectral Energy Retention}
\label{sec:empirical_validation}

The preceding sections define several ways in which the explicit SVDA spectrum could guide intervention, and the empirical evaluation focuses on one: spectral energy retention in the query--spectrum--key score pathway. The objective is to test whether directions assigned low spectral energy can be masked with small changes in predictive performance and then represented by a physically smaller score pathway. The experiments distinguish the masked operator from the dimension-reduced implementation and quantify their agreement under the stated normalization conventions.

The validation focuses on the learned spectral energy of the SVDA diagonal modulation. For each attention head, the learned SVDA spectrum defines an energy distribution over latent score directions. The energy of direction $r$ in head $h$ is defined as
\begin{equation}
E_{h,r}=\sigma_{h,r}^{2}
\end{equation}
This quantity measures the squared magnitude of the learned coefficient assigned to direction $r$ in the score operator. It orders the explicit spectral coefficients, while the realized contribution of a direction also depends on its query and key activations and on downstream network sensitivity. Accordingly, ``high-energy'' and ``low-energy'' refer to the relative values of $\sigma_{h,r}^{2}$ within a head. 
The prescription uses a relative energy-retention rule rather than the absolute value of $\sigma_{h,r}$ or a fixed numerical cutoff such as $\sigma_{h,r}<0.90$. For a prescribed retention ratio $\rho \in (0,1]$, directions within each head are sorted by decreasing energy. The retained set of score-direction indices for head $h$ is denoted by $\mathcal{K}_{h}\subseteq\{1,\ldots,d_k\}$ and is chosen as the smallest set satisfying
\begin{equation}
\frac{\sum_{r \in \mathcal{K}_{h}} E_{h,r}}{\sum_{r=1}^{d_k} E_{h,r}} \geq \rho
\end{equation}
The complement $\{1,\ldots,d_k\}\setminus\mathcal{K}_{h}$ is then removed from the score pathway. Thus, $\rho$ is scale-independent only as a relative within-head retention rule: it specifies the fraction of learned spectral energy to preserve instead of a raw threshold applied to $\sigma$. The absolute score perturbation is not determined by $\rho$ alone. As \eqref{eq:energy_retention_score_bound} shows, the bound also scales with $\|\Sigma\|_F$, so the same value of $\rho$ across two heads or models does not imply the same absolute perturbation. The retained energy fraction remains meaningful as a relative prescription even when the magnitude of the learned $\sigma$ values varies with dataset, optimization trajectory, initialization, or architecture scale.

The empirical protocol separates three related but distinct objects: (a) the original SVDA model, (b) the masked SVDA model, and (c) the dimension-reduced SVDA model. The original model is trained normally with the SVDA attention operator. The masked model sets low-energy entries of the effective diagonal spectrum to zero while leaving the original dense query, key, value, and output projections intact. Comparing its accuracy with the original model tests the empirical effect of masking, while structural savings are realized by the dimension-reduced model. This model physically removes the selected directions from the query projection, key projection, and diagonal spectral modulation while keeping the value and output projections full. This isolates the intervention to score formation: compressing the value or output pathways would change the represented features and would constitute a different intervention.

The resulting intervention can be summarized as follows. For each trained SVDA model and head $h$, the set $\mathcal{K}_{h}$ is computed from the learned $\sigma$ values using $\rho$. The corresponding query and key coordinates and spectral entries are copied into a smaller module; the value and output projections remain full. This produces a dimension-reduced score pathway; the exact relation to masking depends on the normalization convention, which is specified next.

To make this distinction explicit, let $P_{\mathcal{K}_h}$ select the retained coordinates. If $Q$ and $K$ are normalized in the original $d_h$-dimensional space, then the exact coordinate-selection realization is
\begin{equation}
S_h^{\mathrm{exact}}
=\frac{(QP_{\mathcal{K}_h}^{\top})
(P_{\mathcal{K}_h}\Sigma P_{\mathcal{K}_h}^{\top})
(KP_{\mathcal{K}_h}^{\top})^{\top}}{\sqrt{d_h}},
\end{equation}
which is algebraically identical to masking. By contrast, renormalizing the selected rows in the reduced space or replacing $\sqrt{d_h}$ by $\sqrt{k_h}$, where $k_h=|\mathcal{K}_h|$, defines a different operator whose scores can differ from the masked scores. The bounds in Section~\ref{sec:diagnostic_to_prescriptive} apply directly to the masked operator and to the exact coordinate-selection equation above; the reduced-space variant is evaluated empirically because it changes the normalization geometry.

The implementation therefore reports two compressed variants. The first is the norm-preserving diagnostic above: it physically copies the retained query/key rows and spectral entries, but it computes the normalization denominators from the original full-dimensional query/key projections and keeps the original $\sqrt{d_h}$ score scale. This diagnostic matches the masked operator up to floating-point tolerance and is reported as a diagnostic because it retains auxiliary full-dimensional projections for normalization. The second is the deployable reduced-space variant: it copies the retained query/key rows and spectral entries into a smaller dense projection tensor, pads heads to the largest retained width needed for batched computation, applies the retained-coordinate mask, recomputes row normalization within that compressed projection tensor, and keeps the original $\sqrt{d_h}$ score scale. Both variants are used without post-compression fine-tuning. The second variant is the structurally reduced model used for parameter and MAC accounting, and its closeness to the masked prescription is measured empirically.

To quantify the deviation of the implemented dimension-reduced model from the mask, let $f_{\mathrm{masked}}(x)$ and $f_{\mathrm{compressed}}(x)$ denote their output logits. For the same test batches, the relative logit discrepancy is defined as
\begin{equation}
\Delta_{\mathrm{rel}} = \frac{ \lVert f_{\mathrm{masked}}(x)-f_{\mathrm{compressed}}(x)\rVert_{2} }{ \lVert f_{\mathrm{masked}}(x)\rVert_{2}+\epsilon }
\end{equation}
where $\epsilon>0$ is a small numerical stabilizer. Prediction agreement is defined as
\begin{equation}
A_{\mathrm{pred}} = \frac{1}{N} \sum_{i=1}^{N} \mathbb{I} \left[ \arg\max f_{\mathrm{masked}}(x_i) = \arg\max f_{\mathrm{compressed}}(x_i) \right]
\end{equation}
where $N$ is the number of evaluated test samples and $\mathbb{I}[\cdot]$ is the indicator function. Nonzero discrepancy quantifies the operator change introduced by reduced-space normalization; prediction agreement measures practical closeness between the two realizations.

The experiments were conducted on FashionMNIST, CIFAR-10, CIFAR-100, and Food-101~\cite{bossard2014food101}. FashionMNIST, CIFAR-10, and CIFAR-100 use a compact SVDA-ViT configuration, using $32\times32$ configured inputs, patch size 4, four Transformer blocks, four attention heads, embedding dimension 256, batch size 256, and seeds 42, 43, and 44. FashionMNIST and CIFAR-10 were trained for 80 epochs, and CIFAR-100 was trained for 100 epochs. Food-101 was added as a larger natural-image validation experiment, providing 224$\times$224 inputs, patch size 16, 196 visual tokens, four Transformer blocks, four attention heads, embedding dimension 256, batch size 64, 40 training epochs, and the same three seeds. The same score-structural evaluation logic was used for all datasets. All trainable components, including the query, key, value, output, MLP, classifier, and learned spectral coefficients, were optimized jointly from scratch before any intervention was applied.

The primary prescription was the conservative energy-retention rule with $\rho=0.90$, which keeps the smallest number of score directions required to preserve at least 90\% of the learned spectral energy in each head. For every trained model, four quantities were measured: classification accuracy, percentage of score directions removed, reduction in trainable parameters, and reduction in estimated multiply-accumulate operations (MACs). The parameter and MAC reductions are measured relative to the original SVDA Transformer. The MAC estimate counts the projection and attention operations implied by the implemented Transformer blocks and is used as an architecture-level arithmetic estimate rather than as a hardware latency proxy. Because the structural realization compresses only the score pathway and keeps the value/output pathway full, the expected parameter and MAC reductions are modest but structurally meaningful. The experiment tests whether the spectral diagnosis can be translated into meaningful structural simplification while isolating the intervention to the score-forming geometry.

Matched controls were also constructed. For each SVDA energy-retention prescription, two matched masked controls were generated with the same pruning ratio. The random-matched control removes the same number of score directions randomly within the corresponding heads. It is a sanity control for whether mild removal is tolerated without spectral ordering. The largest-$\Sigma$-matched control removes the highest-energy directions instead of the lowest-energy ones. In \tablename~\ref{tab:main-absolute-accuracy}, Random and Largest-$\Sigma$ refer to these matched masked controls. Because masking and deployable reduced-space realization are distinct operators, an additional matched random reduced-space control was also evaluated: for each trained seed, 50 random retained-coordinate sets were converted into the same deployable reduced-space operator as the SVDA prescription, using the same retained counts, normalization convention, padding implementation, and no post-compression fine-tuning. These controls test the learned coefficient ordering under both the functional masked intervention and the deployed reduced-space realization, while the ablation study later examines how coefficient energy relates to single-coordinate functional effects.

The main results for $\rho=0.90$ are shown in \tablename~\ref{tab:main-absolute-accuracy} and \tablename~\ref{tab:main-score-structural}. The accuracy table reports absolute accuracies rather than only deltas. Across the four datasets, the paired mean accuracy change of the deployable dimension-reduced SVDA model relative to its original checkpoint is $+0.003$, $+0.050$, $-0.003$, and $-0.025$ percentage points for FashionMNIST, CIFAR-10, CIFAR-100, and Food-101, respectively. These changes are small relative to the across-seed standard deviations in \tablename~\ref{tab:main-absolute-accuracy}, supporting empirical accuracy preservation in these settings.

\begin{table}[t]
\centering
\caption{Absolute accuracies for the $\rho=0.90$ SVDA intervention. Values are mean $\pm$ sample standard deviation over seeds.}
\label{tab:main-absolute-accuracy}
\small
\begin{tabular}{lccccc}
\toprule
Dataset & Original & Masked & Reduced & Random mask & Largest-$\Sigma$ \\
\midrule
FashionMNIST & $87.06\pm0.35$ & $87.08\pm0.34$ & $87.07\pm0.32$ & $86.65\pm0.36$ & $83.08\pm0.49$ \\
CIFAR-10     & $71.84\pm0.41$ & $71.84\pm0.41$ & $71.89\pm0.41$ & $71.45\pm0.49$ & $69.50\pm0.92$ \\
CIFAR-100    & $40.79\pm0.61$ & $40.80\pm0.63$ & $40.78\pm0.66$ & $39.87\pm0.64$ & $35.90\pm0.24$ \\
Food-101     & $38.03\pm0.50$ & $38.02\pm0.51$ & $38.00\pm0.47$ & $36.60\pm0.56$ & $34.47\pm1.07$ \\
\bottomrule
\end{tabular}
\end{table}

\begin{table}[t]
\centering
\caption{Score-pathway structural reduction and reduced-space diagnostics for $\rho=0.90$. Structural values are mean $\pm$ sample standard deviation over seeds.}
\label{tab:main-score-structural}
\small
\begin{tabular}{lccccc}
\toprule
Dataset & Dir. pruned & Params $\downarrow$ & MACs $\downarrow$ & Rel. L2 & Pred. agree \\
\midrule
FashionMNIST & $24.74\pm0.41$\% & $2.69\pm0.30$\% & $2.93\pm0.33$\% & 0.0096 & 99.72\% \\
CIFAR-10     & $24.51\pm0.93$\% & $2.59\pm0.26$\% & $2.83\pm0.28$\% & 0.0030 & 99.81\% \\
CIFAR-100    & $29.36\pm0.78$\% & $3.78\pm0.15$\% & $4.14\pm0.16$\% & 0.0071 & 99.33\% \\
Food-101     & $53.68\pm3.71$\% & $4.30\pm0.56$\% & $5.38\pm0.70$\% & 0.0223 & 97.66\% \\
\bottomrule
\end{tabular}
\end{table}

The norm-preserving coordinate-selection diagnostic matches the masked model up to floating-point tolerance: the maximum relative logit discrepancy over seeds is at most $1.5\times10^{-8}$ across all datasets, with 100\% prediction agreement. The deployable reduced-space model changes the normalization geometry by recomputing row normalization in the compressed projection tensor; its relative logit discrepancy ranges from 0.0030 to 0.0223 and its prediction agreement ranges from 97.66\% to 99.81\%. These measurements show that the reduced model remains close to the masked prescription while representing a distinct deployable operator.

The repeated random-mask control uses 50 matched random masks per trained seed, with the random draws first summarized within each seed before aggregating across seeds. At $\rho=0.90$, the proposed masked spectral prescription lies at the top of the sampled masked-random distribution for every dataset in these runs: the mean empirical percentile is 1.0 and the mean value of $p(\mathrm{random}\geq\mathrm{spectral})$ is 0.0196. The corresponding random-mask mean accuracy changes are $-0.410$, $-0.391$, $-0.920$, and $-1.429$ percentage points on FashionMNIST, CIFAR-10, CIFAR-100, and Food-101. The matched random reduced-space control is stronger, as expected from the changed normalization geometry, and narrows the gap. \tablename~\ref{tab:main-random-reduced} therefore reports both the mean random reduced-space accuracy and the 5th--95th percentile range of the 50 random reduced-space realizations. The SVDA reduced-space prescription remains above the mean random reduced-space control on all four datasets, with differences of $+0.04$, $+0.05$, $+0.11$, and $+0.44$ percentage points. The empirical probabilities $p(\mathrm{rand}\geq\mathrm{SVDA})$ show that the reduced-space gaps are small on FashionMNIST, CIFAR-10, and CIFAR-100, and clearer on Food-101. This supports the usefulness of the spectral ordering in the present models while distinguishing the functional masked intervention from the deployable reduced-space operator.

\begin{table}[t]
\centering
\caption{Matched random reduced-space control for $\rho=0.90$. Each random retained-coordinate set is converted into the same deployable reduced-space operator as SVDA. The random interval reports the 5th--95th percentile range of the 50 random reduced-space realizations, averaged over seeds.}
\label{tab:main-random-reduced}
\scriptsize
\begin{tabular}{@{}lccccc@{}}
\toprule
Dataset & SVDA reduced & Random mean & Random $q_{05}$--$q_{95}$ & $p(\mathrm{rand}\geq\mathrm{SVDA})$ & SVDA $-$ Random \\
\midrule
FashionMNIST & 87.07 & 87.03 & 86.96--87.10 & 0.222 & +0.04 \\
CIFAR-10     & 71.89 & 71.83 & 71.70--71.95 & 0.261 & +0.05 \\
CIFAR-100    & 40.78 & 40.67 & 40.48--40.85 & 0.203 & +0.11 \\
Food-101     & 38.00 & 37.56 & 37.14--37.90 & 0.046 & +0.44 \\
\bottomrule
\end{tabular}
\end{table}

A second experiment examines how the reported quantities vary with $\rho$. A sensitivity analysis was conducted on CIFAR-10 using $\rho\in\{0.85,0.90,0.95\}$ with three seeds. The results are shown in \tablename~\ref{tab:rho-sensitivity}. Lowering $\rho$ from 0.95 to 0.85 removes more score directions and yields larger parameter and MAC reductions, while the paired mean accuracy change of the reduced-space SVDA model remains between $+0.030$ and $+0.070$ percentage points. The random and largest-$\Sigma$ controls remain below the spectral prescription at all three operating points.

\begin{table}[t]
\centering
\caption{CIFAR-10 sensitivity to the spectral energy-retention parameter $\rho$.}
\label{tab:rho-sensitivity}
\small
\begin{tabular}{lcccccc}
\toprule
$\rho$ & Dir. pruned & Params $\downarrow$ & MACs $\downarrow$ & Reduced $\Delta$ & Random $\Delta$ & Largest $\Delta$ \\
\midrule
0.85 & $32.36\pm1.19$\% & $3.63\pm0.26$\% & $3.96\pm0.28$\% & $+0.070\pm0.017$ & $-0.591\pm0.116$ & $-2.513\pm0.327$ \\
0.90 & $24.51\pm0.93$\% & $2.59\pm0.26$\% & $2.83\pm0.28$\% & $+0.050\pm0.010$ & $-0.391\pm0.094$ & $-2.333\pm0.508$ \\
0.95 & $14.49\pm0.56$\% & $1.38\pm0.30$\% & $1.51\pm0.33$\% & $+0.030\pm0.030$ & $-0.201\pm0.078$ & $-1.960\pm0.588$ \\
\bottomrule
\end{tabular}
\end{table}

The sensitivity study confirms the construction's expected monotonicity: lower $\rho$ removes more score directions and higher $\rho$ removes fewer. The current three operating points position $\rho=0.90$ as a conservative reference point for the main comparison.

To test whether an explicit SVDA spectrum is necessary for score-dimension pruning, a standard-attention baseline was evaluated at the same retained counts induced by the SVDA $\rho=0.90$ prescription. For each standard-attention head, direction $r$ is ranked by the projection-magnitude score $\|W_Q[r]\|_2\|W_K[r]\|_2$, where $W_Q[r]$ and $W_K[r]$ denote the corresponding query and key projection rows. \tablename~\ref{tab:standard-magnitude-baseline} reports this control for FashionMNIST, CIFAR-10, and CIFAR-100.

\begin{table}[t]
\centering
\caption{Standard-attention matched pruning baseline at $\rho=0.90$. Values are mean $\pm$ sample standard deviation over seeds.}
\label{tab:standard-magnitude-baseline}
\small
\begin{tabular}{lccc}
\toprule
Dataset & Original standard & Magnitude-pruned & Random-pruned \\
\midrule
FashionMNIST & $87.35\pm0.14$ & $87.28\pm0.37$ & $86.55\pm0.64$ \\
CIFAR-10     & $73.04\pm0.14$ & $70.80\pm2.17$ & $69.92\pm0.57$ \\
CIFAR-100    & $44.71\pm0.48$ & $42.51\pm1.42$ & $40.21\pm0.53$ \\
\bottomrule
\end{tabular}
\end{table}

This baseline provides an informative comparator. On FashionMNIST, standard-attention magnitude pruning is nearly neutral. On CIFAR-10 and CIFAR-100, it degrades mean accuracy by 2.24 and 2.20 percentage points while still outperforming matched random pruning. Thus, ordinary projection magnitude is a meaningful score-dimension signal. The SVDA contribution is not that $\sigma_r^2$ universally dominates ordinary magnitude as a functional-importance measure, but that the learned spectrum supplies an explicit score-operator coordinate system in which coefficient energy defines a deterministic intervention rule. The functional effect of a retained or removed direction can still depend on query/key activations and downstream sensitivity.

Taken together, the empirical results support a specific operational conclusion: the learned SVDA spectrum can be converted into a deterministic score-pathway intervention; low-energy score directions can be removed with small observed accuracy changes in the evaluated models; the masked prescription can be approximated by a smaller query--spectrum--key pathway; and high-energy directions are consistently more disruptive to remove. This supports the interpretation of $\Sigma$ as an operational control surface for attention, extending SVDA from post-hoc diagnosis to actionable attention interpretability.

The empirical claim concerns structural and arithmetic complexity reduction. The dimension-reduced realization reduces trainable parameters and estimated multiply-accumulate operations by reducing the dimensionality of the score-forming $Q\Sigma K^{\top}$ pathway. These reductions are reported as architecture-level complexity estimates and are independent of hardware-specific kernel efficiency. A hardware-aware implementation would be needed to translate the score-pathway arithmetic savings into optimized latency. If $d_h$ is the original per-head score dimension and $k_h=|\mathcal{K}_h|$ is the retained score dimension after energy retention, then the query-spectrum-key score computation in head $h$ is reduced from order $O(n^2 d_h)$ to $O(n^2 k_h)$, while the value aggregation and output projection remain unchanged.

\section{Limitations}
\label{sec:limitations}

The present evaluation is deliberately controlled. It measures the consequence of applying SVDA energy retention to the score pathway while preserving the value and output pathways. This design isolates the structural role of the learned spectrum and yields modest but interpretable parameter/MAC reductions. Other intervention targets discussed in the framework, including effective-rank control, robustness-guided regularization, and head-redundancy analysis, are important extensions of the same prescriptive mechanism.

The experiments use compact SVDA-ViT configurations, with Food-101 serving as a larger-token natural-image validation scenario. Larger architectures, pretrained backbones, language tasks, and optimized fused implementations will be important for assessing the practical compression and acceleration potential of the same mechanism. The width-matched standard-attention baseline for CIFAR-10 and CIFAR-100 complements the post-training reduction experiment by checking the relationship between SVDA reduction and directly training a narrower score pathway.

The matched random-mask, matched random reduced-space, largest-$\Sigma$, standard-attention magnitude, and width-matched controls test different aspects of the learned spectral ordering. The repeated random-mask results place the spectral prescription above the sampled random-mask distributions in the present compact models. When random retained coordinates are converted into the same deployable reduced-space operator, the random control becomes stronger and the accuracy gap narrows, but the spectral prescription remains above the random reduced-space mean in all four datasets. The largest-$\Sigma$ control provides evidence that removing high-energy coefficients is more disruptive at the tested budgets. The standard-attention magnitude baseline shows that ordinary projection magnitude is also a meaningful score-dimension pruning signal. Broader matched-budget comparisons with specialized Transformer-compression methods will be useful for mapping where explicit spectral coordinates offer the largest advantage over ordinary weight-based pruning or other compression criteria.

Finally, $\sigma_r^2$ measures the magnitude of a learned coefficient in the score operator. The realized contribution of direction $r$ also depends on the query/key activations and downstream network sensitivity. The largest-$\Sigma$ control and the single-direction ablation study in Appendix~\ref{app:functional_ablation} indicate that coefficient ordering carries functional information, but they do not establish $\sigma_r^2$ as a universally superior functional-importance metric. Broader activation-aware importance measures and gated score-direction baselines would further refine this interpretation.

\section{Conclusion}

This paper examined how the explicit spectrum of SVDA can support a transition from diagnosis to intervention. The empirical study focused on spectral energy retention: retaining a prescribed fraction of learned spectral energy and removing the remaining directions from the attention-score pathway. The broader diagnosis--intervention--verification framework also organizes effective-rank targeting, robustness-oriented regularization, and head-redundancy screening as spectrum-guided intervention modes.

Using a spectral energy-retention prescription with $\rho=0.90$, SVDA removed approximately 24.5--53.7\% of attention-score directions across FashionMNIST, CIFAR-10, CIFAR-100, and Food-101. The resulting dimension-reduced models reduced trainable parameters by approximately 2.6--4.3\% and estimated MACs by approximately 2.8--5.4\%, with paired mean accuracy changes between $-0.03$ and $+0.05$ percentage points over three seeds. Exact norm-preserving coordinate selection matched masking up to floating-point tolerance, while the deployable reduced-space model achieved 97.66--99.81\% prediction agreement with the masked model.

The matched controls clarify the structure of the result. SVDA consistently outperformed largest-$\Sigma$ pruning, showing that high-energy spectral directions are structurally important and should not be removed. Against repeated random-matched masking, the conservative energy-retention rule lay in the upper tail of the sampled random distributions in these compact settings. Against matched random reduced-space controls using the same deployable realization, the gap was smaller but the spectral rule remained above the random reduced-space mean on all four datasets. This makes the learned spectrum a useful operational coordinate system for score-pathway simplification.

The CIFAR-10 sensitivity analysis further showed that the retention parameter $\rho$ controls a meaningful compression-conservation trade-off. Lower $\rho$ values produce stronger pruning and larger parameter/MAC reductions, while higher $\rho$ values produce more conservative interventions and tighter masked-versus-compressed agreement. Thus, $\rho$ is not an arbitrary threshold, but an interpretable relative operating parameter. Its absolute perturbation effect still depends on the scale of the learned spectrum.

Overall, the study shows that the learned SVDA spectrum can define a deterministic and structurally implementable intervention on the attention-score pathway. At the conservative budgets examined here, the resulting models exhibit small paired mean accuracy changes with reported across-seed variability. The evidence supports SVDA as a step from intrinsic attention interpretability toward actionable, spectrum-guided model editing.

\appendix

\section{Experimental Details}
\label{app:experimental_details}

This appendix provides additional implementation details for the evaluation reported in Section~\ref{sec:empirical_validation}. The experiments examine whether the learned SVDA spectrum can define a masked score intervention and a corresponding dimension-reduced score pathway. All reported comparisons are made relative to the original trained SVDA model under the same architecture and training configuration.

\subsection{Datasets and Input Settings}
\label{app:datasets}

Four image-classification datasets were used, including FashionMNIST, CIFAR-10, CIFAR-100, and Food-101. \tablename~\ref{tab:app-datasets} summarizes the corresponding class counts, input resolutions, patch sizes, and image-token counts. The token count excludes any optional classification token and refers only to image patch tokens. FashionMNIST, CIFAR-10, and CIFAR-100 use compact $32\times32$ inputs with patch size 4, producing 64 image tokens. Food-101 is included as a larger natural-image validation scenario with $224\times224$ inputs and patch size 16, producing 196 image tokens. Food-101 is used to test whether the score-structural prescription remains stable in a higher-resolution natural-image setting with a compact SVDA-ViT trained from scratch.

\begin{table}[ht]
\centering
\caption{Datasets and input scenarios.}
\label{tab:app-datasets}
\begin{tabular}{lrrrr}
\toprule
Dataset & Classes & Input size & Patch size & Image tokens \\
\midrule
FashionMNIST & 10  & $32\times32$   & 4  & 64  \\
CIFAR-10     & 10  & $32\times32$   & 4  & 64  \\
CIFAR-100    & 100 & $32\times32$   & 4  & 64  \\
Food-101     & 101 & $224\times224$ & 16 & 196 \\
\bottomrule
\end{tabular}
\end{table}

\subsection{Model Configuration}
\label{app:model_config}

All experiments used compact SVDA-ViT configurations with four Transformer blocks, four attention heads, and embedding dimension 256. \tablename~\ref{tab:app-model-config} summarizes the model and training configuration used for each dataset. The same score-compression logic was used across datasets. The Food-101 setting differs only in input resolution, patch size, batch size, number of classes, and number of training epochs. The Food-101 experiment uses a smaller batch size because of the larger $224\times224$ input resolution. All datasets use the same three random seeds. Reported values are means over these three seeds.

\begin{table}[ht]
\centering
\caption{Model and training configurations used in the experiments.}
\label{tab:app-model-config}
\begin{tabular}{lrrrrrrr}
\toprule
Dataset & Blocks & Heads & Embed dim. & Patch & Batch & Epochs & Seeds \\
\midrule
FashionMNIST & 4 & 4 & 256 & 4  & 256 & 80  & 42, 43, 44 \\
CIFAR-10     & 4 & 4 & 256 & 4  & 256 & 80  & 42, 43, 44 \\
CIFAR-100    & 4 & 4 & 256 & 4  & 256 & 100 & 42, 43, 44 \\
Food-101     & 4 & 4 & 256 & 16 & 64  & 40  & 42, 43, 44 \\
\bottomrule
\end{tabular}
\end{table}

The complete run configuration was recorded with each result file. The query, key, value, output, MLP, classifier, and $\Sigma$ parameters were trained jointly; only the post-training interventions changed the score pathway. The spectral coefficients were initialized to one. Base runs used the Adam optimizer with learning rate $3\times10^{-4}$, no optimizer weight decay, no learning-rate schedule, gradient clipping at 1.0, CUDA automatic mixed precision, deterministic seeding, no dropout, no positional embedding, and the fixed training lengths in \tablename~\ref{tab:app-model-config}. The base energy-retention experiments used $\ell_1$ spectrum penalty $\lambda=0$; Appendix~\ref{app:l1_regularization} reports the separate $\ell_1$ spectrum-regularization check with $\lambda=10^{-4}$. The base training configuration held the orthogonality coefficient fixed at $10^{-3}$ across compared models. Training augmentation used random cropping for FashionMNIST, random cropping plus horizontal flipping for CIFAR-10 and CIFAR-100, and random resized crops plus horizontal flipping for Food-101; evaluation used the corresponding normalized test split and a center crop for Food-101. The saved final checkpoint for each seed was used for the reported post-training masked, norm-preserving, reduced-space, random-mask, random reduced-space, largest-$\Sigma$, and magnitude-pruning evaluations. No post-compression fine-tuning was applied. The main retention value $\rho=0.90$ was fixed before the main comparison; the test split was not used to choose $\rho$, and the additional $\rho$ values were used only for the CIFAR-10 sensitivity study.

\subsection{Energy-Retention Prescription}
\label{app:energy_retention}

For each attention head $h$, the learned SVDA spectrum defines a directional energy
\begin{equation*}
E_{h,r} = \sigma_{h,r}^{2}.
\end{equation*}
For a retention ratio $\rho$, directions are sorted by decreasing energy and the smallest retained set $\mathcal{K}_{h}$ is selected such that
\begin{equation*}
\frac{\sum_{r\in\mathcal{K}_{h}}E_{h,r}}{\sum_{r}E_{h,r}} \geq \rho.
\end{equation*}
All directions outside $\mathcal{K}_{h}$ are removed from the score pathway. The main experiments use $\rho=0.90$. The sensitivity experiment on CIFAR-10 additionally evaluates $\rho\in\{0.85,0.90,0.95\}$. The parameter $\rho$ is a relative within-head energy-retention ratio rather than a raw threshold on $\sigma$. Thus, the same value of $\rho$ can be applied across datasets and heads as a relative prescription even when the absolute scale of the learned spectra differs. It should not be interpreted as fixing the same absolute perturbation across heads or models, because the bound in \eqref{eq:energy_retention_score_bound} also depends on $\|\Sigma\|_F$.

\subsection{Masked and Dimension-Reduced Realizations}
\label{app:masked_compressed}

The empirical protocol distinguishes three models:

\begin{itemize}
\item \textbf{Original SVDA model}: the trained model before intervention.
\item \textbf{Masked SVDA model}: a functional intervention in which low-energy entries of the effective spectrum are set to zero while the original dense tensors remain present.
\item \textbf{Norm-preserving diagnostic}: a coordinate-selection implementation that keeps auxiliary full-dimensional query/key projections only to preserve the original normalization denominators.
\item \textbf{Dimension-reduced SVDA model}: a deployable implementation in which the retained score directions are physically copied into smaller query and key score projections together with the retained entries of $\Sigma$.
\end{itemize}

The dimension-reduced realization modifies only the score-forming pathway, namely $Q$, $K$, and $\Sigma$. The value and output projections are preserved. Exact equivalence to masking additionally requires normalization in the original space, coordinate selection without reduced-space renormalization, and retention of the original $1/\sqrt{d_h}$ scale. Reduced-space renormalization or $1/\sqrt{k_h}$ scaling defines a different, dimension-reduced operator.

Let $N_D(u)=u/(\|u\|_2+\epsilon)$ denote row-wise normalization over a $D$-dimensional score coordinate vector. For head $h$, write $\widehat{Q}_h=N_{d_h}(Q_h)$ and $\widehat{K}_h=N_{d_h}(K_h)$. The masked score operator is
\begin{equation*}
S_h^{\mathrm{mask}}
=\frac{\widehat{Q}_h\left(M_h\Sigma_h\right)\widehat{K}_h^{\top}}{\sqrt{d_h}},
\end{equation*}
where $M_h$ is the binary retained-coordinate mask. The norm-preserving coordinate-selection diagnostic uses
\begin{equation*}
S_h^{\mathrm{exact}}
=\frac{(\widehat{Q}_hP_{\mathcal{K}_h}^{\top})
\left(P_{\mathcal{K}_h}\Sigma_hP_{\mathcal{K}_h}^{\top}\right)
(\widehat{K}_hP_{\mathcal{K}_h}^{\top})^{\top}}{\sqrt{d_h}},
\end{equation*}
with the norms computed before coordinate selection. This is the diagnostic used to verify exact masked-versus-coordinate-selection agreement. The deployable reduced-space implementation instead computes the query/key normalization inside the compressed projection tensor. Because the retained width can differ by layer and head, the implementation uses a shared active width $k_{\star}$ equal to the largest retained width required by any head in that compressed model and masks inactive padded coordinates. Its score computation is therefore
\begin{equation*}
S_h^{\mathrm{red}}
=\frac{N_{k_{\star}}\!\left(Q_hP_{\mathcal{K}_h}^{\top}\right)
\left(P_{\mathcal{K}_h}\Sigma_hP_{\mathcal{K}_h}^{\top}\right)
N_{k_{\star}}\!\left(K_hP_{\mathcal{K}_h}^{\top}\right)^{\top}}{\sqrt{d_h}},
\end{equation*}
up to the zero padding used for dense batched tensors. This equation defines the deployable reduced-space operator evaluated empirically in the experiments.

\subsection{Masked--Compressed Agreement Diagnostics}
\label{app:equivalence}

To assess whether the dimension-reduced model remains close to the masked prescription, the logits of the masked model and the dimension-reduced model are compared on the same test batches. The relative logit discrepancy is
\begin{equation*}
\Delta_{\mathrm{rel}} = \frac{\left\lVert f_{\mathrm{masked}}(x)-f_{\mathrm{compressed}}(x)\right\rVert_2} {\left\lVert f_{\mathrm{masked}}(x)\right\rVert_2+\epsilon}
\end{equation*}
where $\epsilon$ is a small numerical stabilizer. Prediction agreement is
\begin{equation*}
A_{\mathrm{pred}} = \frac{1}{N} \sum_{i=1}^{N} \mathbb{I} \left[ \arg\max f_{\mathrm{masked}}(x_i) = \arg\max f_{\mathrm{compressed}}(x_i) \right]
\end{equation*}

Lower relative discrepancy and higher prediction agreement indicate empirical closeness. Exact equivalence corresponds to score/logit differences at floating-point tolerance; the reported nonzero discrepancies quantify the practical effect of reduced-space normalization.

\subsection{MAC and Parameter Accounting}
\label{app:mac_parameter_accounting}

Parameter reductions are computed from trainable parameter counts before and after score-structural compression. MAC reductions are estimated from the projection and attention operations implied by the implemented Transformer blocks. The MAC estimate is used as an architecture-level arithmetic proxy rather than a hardware latency proxy. For a dense SVDA attention head, the score-forming projections and attention-score computation include contributions from the query projection, the key projection, the diagonal spectrum, and the query--spectrum--key matrix product. In the dimension-reduced realization, each head $h$ retains only $k_h=|\mathcal{K}_h|$ score directions. Thus, the score-forming part of the computation is reduced from dimension $d_h$ to $k_h$ for the $Q\Sigma K^{\top}$ pathway. The value and output pathways remain full by design. The total reported MAC reduction therefore reflects the arithmetic savings induced by the smaller score pathway. Fused kernels, sparse-kernel acceleration, and deployment-level optimization are implementation-dependent concerns and are treated separately from the architecture-level accounting.

Let $H$ denote the number of attention heads and $d_h$ the dimension of each head, so that $d_{\mathrm{model}}=Hd_h$. Relative to plain attention with the same query, key, value, and output projections, dense SVDA adds $Hd_h=d_{\mathrm{model}}$ learned spectral coefficients per layer. Applying the diagonal spectrum requires $nd_{\mathrm{model}}$ elementwise multiplications per layer. Row-normalizing both the query and key projections additionally requires $O(nd_{\mathrm{model}})$ arithmetic for squared norms, reductions, reciprocal square roots, and rescaling. These costs are lower order in the token count than the $O(n^2d_{\mathrm{model}})$ query-key score computation, but they are not zero. The reductions reported for the dimension-reduced SVDA model are measured relative to dense SVDA; a comparison with plain attention must subtract this additional spectral and normalization overhead when reporting the net parameter and arithmetic effect.

\tablename~\ref{tab:app-svda-overhead} reports the measured dense SVDA overhead relative to a plain-attention model with the same projections, MLPs, classifier head, and input resolution. The parameter overhead is 1,024 parameters in all configurations, corresponding to four layers with $d_{\mathrm{model}}=256$ learned spectral coefficients per layer. The MAC overhead includes diagonal modulation and the row-normalization arithmetic used by the implementation.

\begin{table}[ht]
\centering
\caption{Dense SVDA overhead relative to plain attention. MACs are per image under the accounting convention used for the reported reductions.}
\label{tab:app-svda-overhead}
\small
\begin{tabular}{lrrrrrr}
\toprule
Dataset & Std. params & SVDA params & $\Delta$ params & Std. MACs & SVDA MACs & $\Delta$ MACs \\
\midrule
FashionMNIST & 3,166,730 & 3,167,754 & 1,024 & 213,389,824 & 213,724,704 & 334,880 \\
CIFAR-10     & 3,174,922 & 3,175,946 & 1,024 & 213,914,112 & 214,248,992 & 334,880 \\
CIFAR-100    & 3,198,052 & 3,199,076 & 1,024 & 213,937,152 & 214,272,032 & 334,880 \\
Food-101     & 3,382,629 & 3,383,653 & 1,024 & 737,750,272 & 738,765,216 & 1,014,944 \\
\bottomrule
\end{tabular}
\end{table}

\subsection{Aggregate Empirical Results}
\label{app:aggregate_results}

\tablename~\ref{tab:app-main-results} reports the paired accuracy differences for the main $\rho=0.90$ SVDA experiments. Values are means $\pm$ sample standard deviations over the three seeds. Positive values indicate higher accuracy than the original trained SVDA checkpoint for the same seed.

\begin{table}[ht]
\centering
\caption{Paired accuracy differences for the $\rho=0.90$ SVDA intervention. Values are percentage points.}
\label{tab:app-main-results}
\small
\begin{tabular}{lrrrr}
\toprule
Dataset & Masked & Reduced & Random & Largest-$\Sigma$ \\
\midrule
FashionMNIST & $+0.013\pm0.015$ & $+0.003\pm0.040$ & $-0.410\pm0.042$ & $-3.987\pm0.813$ \\
CIFAR-10     & $+0.003\pm0.015$ & $+0.050\pm0.010$ & $-0.391\pm0.094$ & $-2.333\pm0.508$ \\
CIFAR-100    & $+0.013\pm0.015$ & $-0.003\pm0.042$ & $-0.920\pm0.084$ & $-4.890\pm0.419$ \\
Food-101     & $-0.004\pm0.024$ & $-0.025\pm0.034$ & $-1.429\pm0.174$ & $-3.558\pm0.632$ \\
\bottomrule
\end{tabular}
\end{table}

\tablename~\ref{tab:app-seed-level} gives the seed-level accuracies underlying the main aggregate results. The random column is the within-seed mean over the 50 matched random masks.

\begin{table}[!tb]
\centering
\caption{Seed-level accuracies for the main $\rho=0.90$ SVDA experiments.}
\label{tab:app-seed-level}
\small
\begin{tabular}{llrrrrr}
\toprule
Dataset & Seed & Original & Masked & Reduced & Random & Largest-$\Sigma$ \\
\midrule
FashionMNIST & 42 & 87.08 & 87.08 & 87.12 & 86.72 & 83.29 \\
FashionMNIST & 43 & 86.71 & 86.74 & 86.72 & 86.27 & 83.42 \\
FashionMNIST & 44 & 87.40 & 87.41 & 87.36 & 86.98 & 82.52 \\
CIFAR-10 & 42 & 72.13 & 72.12 & 72.17 & 71.76 & 70.18 \\
CIFAR-10 & 43 & 71.37 & 71.37 & 71.42 & 70.88 & 68.46 \\
CIFAR-10 & 44 & 72.01 & 72.03 & 72.07 & 71.70 & 69.87 \\
CIFAR-100 & 42 & 40.08 & 40.08 & 40.03 & 39.12 & 35.64 \\
CIFAR-100 & 43 & 41.20 & 41.21 & 41.23 & 40.22 & 35.93 \\
CIFAR-100 & 44 & 41.08 & 41.11 & 41.09 & 40.26 & 36.12 \\
Food-101 & 42 & 38.58 & 38.59 & 38.52 & 37.24 & 35.69 \\
Food-101 & 43 & 37.61 & 37.62 & 37.62 & 36.29 & 33.96 \\
Food-101 & 44 & 37.89 & 37.86 & 37.87 & 36.26 & 33.75 \\
\bottomrule
\end{tabular}
\end{table}

\tablename~\ref{tab:app-main-structure} repeats the structural reductions and masked--reduced diagnostics for completeness.

\begin{table}[!tb]
\centering
\caption{Structural reduction and masked--reduced diagnostics for the $\rho=0.90$ SVDA intervention.}
\label{tab:app-main-structure}
\small
\begin{tabular}{lccccc}
\toprule
Dataset & Dir. pruned & Params $\downarrow$ & MACs $\downarrow$ & Rel. L2 & Pred. agree \\
\midrule
FashionMNIST & $24.74\pm0.41$\% & $2.69\pm0.30$\% & $2.93\pm0.33$\% & 0.0096 & 99.72\% \\
CIFAR-10     & $24.51\pm0.93$\% & $2.59\pm0.26$\% & $2.83\pm0.28$\% & 0.0030 & 99.81\% \\
CIFAR-100    & $29.36\pm0.78$\% & $3.78\pm0.15$\% & $4.14\pm0.16$\% & 0.0071 & 99.33\% \\
Food-101     & $53.68\pm3.71$\% & $4.30\pm0.56$\% & $5.38\pm0.70$\% & 0.0223 & 97.66\% \\
\bottomrule
\end{tabular}
\end{table}

\subsection{Control Experiments}
\label{app:controls}

Two matched masked controls were used at the same pruning ratio as the SVDA energy-retention prescription. The random-matched control removes the same number of score directions as the SVDA prescription, but selects directions randomly within each head. The largest-$\Sigma$ control removes the highest-energy score directions instead of the lowest-energy directions. \tablename~\ref{tab:app-controls} reports the matched-control results used in the main-text discussion for the $\rho=0.90$ experiments. Values denote mean accuracy differences in percentage points. Positive values indicate that the reduced-space SVDA energy-retention model gives higher accuracy than the corresponding masked control.

\begin{table}[!tb]
\centering
\caption{Matched masked-control results for $\rho=0.90$.}
\label{tab:app-controls}
\begin{tabular}{lrr}
\toprule
Dataset & SVDA $-$ Random & SVDA $-$ Largest-$\Sigma$ \\
\midrule
FashionMNIST & +0.413 pp & +3.990 pp \\
CIFAR-10     & +0.441 pp & +2.383 pp \\
CIFAR-100    & +0.917 pp & +4.887 pp \\
Food-101     & +1.404 pp & +3.533 pp \\
\bottomrule
\end{tabular}
\end{table}

\tablename~\ref{tab:app-random-distribution} reports the repeated random-mask distribution. The spectral column uses the masked spectral prescription, because the random masks are functional masked interventions matched to the retained counts.

\begin{table}[!tb]
\centering
\caption{Repeated matched random-mask distribution for $\rho=0.90$. Values are accuracies; $q_{05}$--$q_{95}$ denotes the 5th--95th percentile range of the 50-mask distribution, averaged over seeds.}
\label{tab:app-random-distribution}
\small
\begin{tabular}{lrrrr}
\toprule
Dataset & Spectral masked & Random mean & Random $q_{05}$--$q_{95}$ & $p(\mathrm{rand}\geq\mathrm{spec})$ \\
\midrule
FashionMNIST & 87.08 & 86.65 & 86.42--86.86 & 0.0196 \\
CIFAR-10     & 71.84 & 71.45 & 71.23--71.64 & 0.0196 \\
CIFAR-100    & 40.80 & 39.87 & 39.46--40.26 & 0.0196 \\
Food-101     & 38.02 & 36.60 & 36.07--37.06 & 0.0196 \\
\bottomrule
\end{tabular}
\end{table}

\tablename~\ref{tab:app-random-reduced-distribution} reports the matched random reduced-space control. For each trained seed, 50 random retained-coordinate sets were sampled with the same retained count as the SVDA prescription in each head. Each random set was then instantiated as the same deployable reduced-space operator used for the SVDA reduced model, with the same padding and normalization implementation and without post-compression fine-tuning. This separates functional masking from reduced-space realization.

\begin{table}[!tb]
\centering
\caption{Repeated matched random reduced-space distribution for $\rho=0.90$. Values are accuracies; $q_{05}$--$q_{95}$ denotes the 5th--95th percentile range of the 50 reduced-space random realizations, averaged over seeds.}
\label{tab:app-random-reduced-distribution}
\small
\begin{tabular}{lrrrr}
\toprule
Dataset & SVDA reduced & Random mean & Random $q_{05}$--$q_{95}$ & $p(\mathrm{rand}\geq\mathrm{SVDA})$ \\
\midrule
FashionMNIST & 87.07 & 87.03 & 86.96--87.10 & 0.222 \\
CIFAR-10     & 71.89 & 71.83 & 71.70--71.95 & 0.261 \\
CIFAR-100    & 40.78 & 40.67 & 40.48--40.85 & 0.203 \\
Food-101     & 38.00 & 37.56 & 37.14--37.90 & 0.046 \\
\bottomrule
\end{tabular}
\end{table}

\subsection{Energy-Retention Sensitivity}
\label{app:rho_sensitivity}

The CIFAR-10 sensitivity experiment evaluates the retention parameter $\rho$ at $\rho\in\{0.85,0.90,0.95\}$. \tablename~\ref{tab:app-rho-sensitivity} reports the corresponding score-structural results.

\begin{table}[!tb]
\centering
\caption{CIFAR-10 sensitivity to $\rho$.}
\label{tab:app-rho-sensitivity}
\small
\begin{tabular}{lcccc}
\toprule
$\rho$ & Dir. pruned & Params $\downarrow$ & MACs $\downarrow$ & Reduced $\Delta$ \\
\midrule
0.85 & $32.36\pm1.19$\% & $3.63\pm0.26$\% & $3.96\pm0.28$\% & $+0.070\pm0.017$ \\
0.90 & $24.51\pm0.93$\% & $2.59\pm0.26$\% & $2.83\pm0.28$\% & $+0.050\pm0.010$ \\
0.95 & $14.49\pm0.56$\% & $1.38\pm0.30$\% & $1.51\pm0.33$\% & $+0.030\pm0.030$ \\
\bottomrule
\end{tabular}
\end{table}

\tablename~\ref{tab:app-rho-controls} reports the matched-control results for the same CIFAR-10 sensitivity experiment. Values denote mean accuracy differences in percentage points.

\begin{table}[!tb]
\centering
\caption{CIFAR-10 matched-control results.}
\label{tab:app-rho-controls}
\begin{tabular}{lrr}
\toprule
$\rho$ & SVDA $-$ Random & SVDA $-$ Largest-$\Sigma$ \\
\midrule
0.85 & +0.661 pp & +2.583 pp \\
0.90 & +0.441 pp & +2.383 pp \\
0.95 & +0.231 pp & +1.990 pp \\
\bottomrule
\end{tabular}
\end{table}

\subsection{Standard-Attention and Width-Matched Baselines}
\label{app:standard_baselines}

\tablename~\ref{tab:app-standard-magnitude} repeats the standard-attention magnitude-pruning baseline with paired differences for the datasets used in this control.

\begin{table}[ht]
\centering
\caption{Standard-attention magnitude-pruning baseline at matched $\rho=0.90$ retained counts.}
\label{tab:app-standard-magnitude}
\small
\begin{tabular}{lrrrrr}
\toprule
Dataset & Original & Magnitude & Mag. $\Delta$ & Random & Rand. $\Delta$ \\
\midrule
FashionMNIST & $87.35\pm0.14$ & $87.28\pm0.37$ & $-0.073\pm0.275$ & $86.55\pm0.64$ & $-0.803\pm0.556$ \\
CIFAR-10     & $73.04\pm0.14$ & $70.80\pm2.17$ & $-2.237\pm2.068$ & $69.92\pm0.57$ & $-3.120\pm0.482$ \\
CIFAR-100    & $44.71\pm0.48$ & $42.51\pm1.42$ & $-2.197\pm1.135$ & $40.21\pm0.53$ & $-4.493\pm0.102$ \\
\bottomrule
\end{tabular}
\end{table}

\tablename~\ref{tab:app-width-matched} reports a from-scratch width-matched standard-attention baseline. The narrow model uses $d_k=48$, which is close to the retained score width induced by the $\rho=0.90$ SVDA prescription on CIFAR-10 and CIFAR-100. This comparison checks whether post-training score-pathway reduction should be understood alongside a model trained with a narrower score pathway from the outset.

\begin{table}[ht]
\centering
\caption{From-scratch width-matched standard-attention baseline. Values are accuracies.}
\label{tab:app-width-matched}
\small
\begin{tabular}{lrrrr}
\toprule
Dataset & Standard $d_k=64$ & Standard $d_k=48$ & Dense SVDA & Reduced SVDA \\
\midrule
CIFAR-10  & $73.04\pm0.14$ & $72.34\pm0.29$ & $71.84\pm0.41$ & $71.89\pm0.41$ \\
CIFAR-100 & $44.71\pm0.48$ & $44.20\pm0.36$ & $40.79\pm0.61$ & $40.78\pm0.66$ \\
\bottomrule
\end{tabular}
\end{table}

\subsection{Spectrum-Regularization Check}
\label{app:l1_regularization}

As a second intervention check, CIFAR-10 SVDA models were trained with an $\ell_1$ penalty on the learned spectra, $\lambda\sum_{\ell,h}\|\Sigma^{(\ell,h)}\|_1$, using $\lambda=10^{-4}$. \tablename~\ref{tab:app-l1} shows that this regularizer increases the fraction of removable score directions at $\rho=0.90$, with the corresponding accuracy behavior reported in the final column.

\begin{table}[ht]
\centering
\caption{CIFAR-10 $\ell_1$ spectrum-regularization check at $\rho=0.90$.}
\label{tab:app-l1}
\small
\begin{tabular}{lrrrrr}
\toprule
$\lambda$ & Original acc. & Dir. pruned & Params $\downarrow$ & MACs $\downarrow$ & Reduced $\Delta$ \\
\midrule
0 & $71.84\pm0.41$ & $24.51\pm0.93$\% & $2.59\pm0.26$\% & $2.83\pm0.28$\% & $+0.050\pm0.010$ \\
$10^{-4}$ & $71.99\pm0.97$ & $47.10\pm1.77$\% & $4.15\pm0.52$\% & $4.52\pm0.57$\% & $+0.027\pm0.071$ \\
\bottomrule
\end{tabular}
\end{table}

\subsection{Single-Direction Functional Ablation}
\label{app:functional_ablation}

To test how $\sigma_r^2$ relates to functional perturbation, a single-coordinate ablation study was run on CIFAR-10 for seed 42. Each score direction was ablated individually and compared with validation-loss change, accuracy change, and relative logit change. \tablename~\ref{tab:app-ablation-correlations} reports Pearson and Spearman correlations for both $\sigma_r^2$ and a Q/K projection-magnitude score. The correlations show that $\sigma_r^2$ carries functional information, especially in rank correlation with logit perturbation, but they do not identify it as a universally superior importance criterion: Q/K magnitude is comparable and stronger for some metrics. The role of $\sigma_r^2$ in this work is therefore as an intrinsic coordinate-energy signal within the SVDA score operator, while functional importance can also depend on activations and downstream sensitivity.

\begin{table}[ht]
\centering
\caption{CIFAR-10 seed-42 single-direction ablation correlations.}
\label{tab:app-ablation-correlations}
\small
\begin{tabular}{lrr}
\toprule
Comparison & Pearson $r$ & Spearman $\rho$ \\
\midrule
$\sigma^2$ vs. $\Delta$ loss & 0.452 & 0.548 \\
$\sigma^2$ vs. relative logit L2 & 0.616 & 0.905 \\
$\sigma^2$ vs. $\Delta$ accuracy & -0.007 & -0.128 \\
Q/K magnitude vs. $\Delta$ loss & 0.619 & 0.526 \\
Q/K magnitude vs. relative logit L2 & 0.848 & 0.939 \\
Q/K magnitude vs. $\Delta$ accuracy & 0.093 & -0.101 \\
\bottomrule
\end{tabular}
\end{table}

\bibliographystyle{unsrt}  
\bibliography{references}

\end{document}